\documentclass[preprint]{iucr}
\pdfoutput=1 

\usepackage{amsthm}
\usepackage{amsmath}
\usepackage{amssymb}
\usepackage{bm}
\usepackage{mathrsfs}
\usepackage[dvips]{graphicx}
\usepackage{algorithm}
\usepackage[noend]{algpseudocode}

\usepackage{color}

\usepackage{url}

\usepackage{supertabular}

\newtheorem{thm}{Theorem}

\newtheorem{lemma}[thm]{Lemma}

\theoremstyle{definition}

\theoremstyle{remark}

\numberwithin{thm}{section}

\DeclareMathAlphabet{\mathsfsl}{OT1}{cmss}{m}{sl}

\renewcommand{\phi}{\varphi}

\newcommand{\argmin}{\operatorname*{arg\; min}}

\newcommand{\Expect}{\operatorname{\mathbb{E}}}

\newcommand{\diag}{\operatorname{diag}}

\newcommand{\bl}{\boldsymbol{l}}

\newcommand{\bZ}{\boldsymbol{Z}}

\newcommand{\bX}{\boldsymbol{X}}
\newcommand{\bE}{\boldsymbol{E}}
\newcommand{\bF}{\boldsymbol{F}}

\newcommand{\bO}{\boldsymbol{O}}
\newcommand{\bP}{\boldsymbol{P}}

\newcommand{\bSigma}{\boldsymbol\Sigma}

\newcommand{\bU}{\boldsymbol{U}}
\newcommand{\bV}{\boldsymbol{V}}

\def\reals{\mathbb{R}}

\def\b0{\mathbf{0}}
\def\bP{\boldsymbol{P}}

\def\bSigma{\boldsymbol\Sigma}

\def\bU{\boldsymbol{U}}
\def\bW{\boldsymbol{W}}
\def\bT{\boldsymbol{T}}

\def\bC{\boldsymbol{C}}

\def\bA{\boldsymbol{A}}
\def\bY{\boldsymbol{Y}}
\def\bB{\boldsymbol{B}}

\def\bI{\mathbf{I}}

\usepackage[english]{babel}
\usepackage[utf8]{inputenc}
\usepackage{amsmath}
\usepackage{graphicx}
\usepackage{color}
\usepackage{algorithm}
\date{\today}
\begin{document}

\title{Anisotropic twicing for single particle reconstruction using 
autocorrelation analysis}
\maketitle

\begin{abstract}
The missing phase problem in X-ray crystallography is commonly solved using the
technique of molecular replacement \cite{rossmann62, rossmann01, scapin13}, 
which borrows phases from a previously
solved homologous structure, and appends them to the measured Fourier magnitudes
of the diffraction patterns of the unknown structure. More recently, molecular 
replacement has been proposed for solving the missing orthogonal matrices 
problem arising in Kam's autocorrelation analysis \cite{kam1977, kam1980} for 
single particle reconstruction using X-ray free electron lasers 
\cite{Saldin2009, Hosseinizadeh2015, Starodub12ncom} and cryo-EM \cite{Bhamre2014}. In classical
molecular replacement, it is common to estimate the magnitudes of the unknown 
structure as twice the measured magnitudes minus the magnitudes of
the homologous structure, a procedure known as 
`twicing' \cite{tukey77}. Mathematically, this is equivalent to finding an 
unbiased estimator for
a complex-valued scalar \cite{Main1979}. We generalize this scheme for the
case of estimating real or complex valued matrices arising in single particle 
autocorrelation analysis. We name this approach ``Anisotropic Twicing'' 
because unlike the scalar case, the unbiased estimator is not obtained by a 
simple magnitude isotropic correction. We compare the performance of the least 
squares, twicing and anisotropic twicing estimators on synthetic and 
experimental datasets. We demonstrate 3D homology modeling in cryo-EM directly from experimental data without iterative refinement or class averaging, for the first time.
\end{abstract}

\section{Introduction}
The missing phase problem in crystallography entails recovering information
about a crystal structure that is lost during the process of imaging. In X-ray
crystallography, the measured diffraction patterns provide information about the
modulus of the 3D Fourier transform of the crystal. The phases of the Fourier
coefficients need to be recovered by other means, in order to reconstruct the 3D
electron density map of the crystal. 
A popular method to solve the missing phase problem
is Molecular Replacement (MR) \cite{rossmann62, rossmann01, scapin13}, which 
relies on a previously
solved homologous structure which is similar to the unknown structure. The 
unknown
structure is then estimated using the Fourier magnitudes of its diffraction
data, along with phases from the homologous
structure. 

The missing phase problem can be formulated mathematically using matrix notation that enables generalization as follows. Each 
Fourier coefficient $\bA$ is a complex-valued scalar, i.e., 
$\bA\in\mathbb{C}^{1\times 1}$ that we wish to estimate, given measurements of
$\bC=\bA\bA^*$ ($\bA^*$ denotes the complex conjugate transpose of $\bA$, i.e., 
$\bA^*_{ij} = \overline{A_{ji}})$, corresponding to the Fourier squared magnitudes, and 
$\bB$ corresponds to a
previously solved homologous structure such that $\bA=\bB+\bE$, where $\bE$ is 
a small perturbation. We denote an estimator of $\bA$ as $\hat\bA$. There
are many possible choices for such an estimator. One such choice is the solution
to the least squares problem
\begin{equation}\label{eq:leastsquareintro}
\hat{\bA}_{\text{LS}}=\argmin_{\bA}\|\bA-\bB\|_F,\,\,\,\text{subject to 
}\bA\bA^*=\bC
\end{equation}
where $||.||_F$ denotes the Frobenius norm.
However, it has been noticed that $\hat{\bA}_{\text{LS}}$ does not reveal the 
correct relative magnitude of the unknown part of the crystal structure, and 
the recovered magnitude is about half of the actual value. As a magnitude 
correction scheme, it was empirically found that setting the magnitude to be 
twice the experimentally measured magnitude minus the magnitude of the 
homologous structure has the desired effect of approximately resolving the 
issue.  That is, the estimator $2\hat{\bA}_{\text{LS}}-\bB$ is used instead. The theoretical 
advantage of this unbiased estimator for the case when 
$\bA\in\mathbb{C}^{1\times 1}$ has been justified in
~\cite{Main1979}. Following \cite{tukey77}, we refer to this procedure as 
twicing.

The advantage of using twicing is demonstrated in the following illustrative 
toy experiment \cite{CowtanCat} for the 2D case. We start with an image of a 
cat with a tail, which is the unknown image that we want to recover. We are 
given the Fourier magnitudes of the unknown image, measured in an experiment. 
In analogy with a known homologous structure used in MR, we have access to a 
similar image, that of a cat, but with its tail missing. We show the results of 
retrieving the original image using least squares, with and without employing 
twicing for magnitude correction, and note that twicing restores the tail 
better than least squares (see Fig~\ref{fig:tail_cat}).

\begin{figure}[]

\centering
\begin{tabular}{cc}
\includegraphics[width=0.3\linewidth]{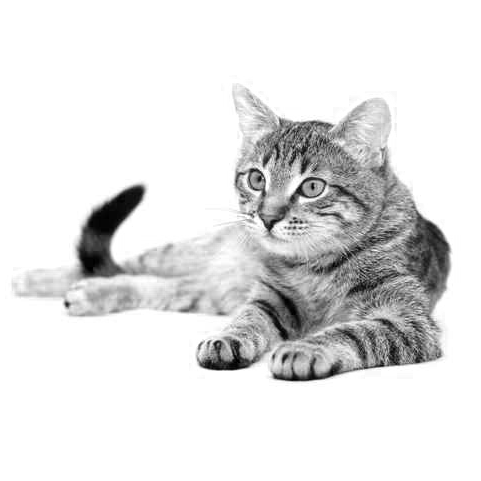}
&\includegraphics[width=0.3\linewidth]{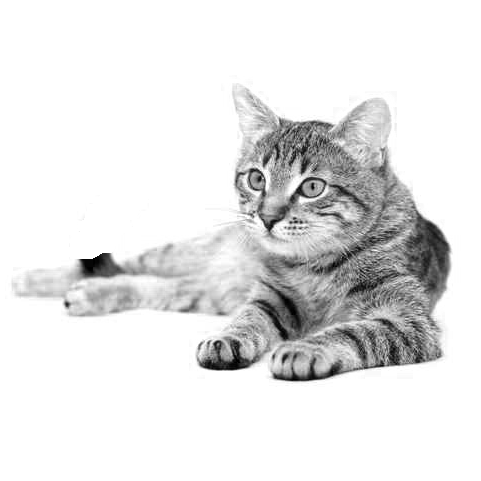}\\
(a) & (b)\\
\includegraphics[width=0.3\linewidth]{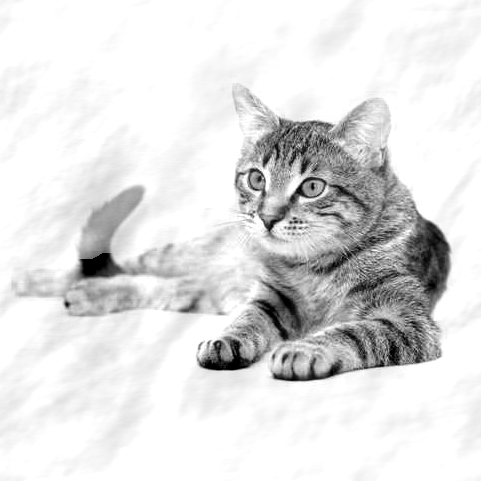}
&\includegraphics[width=0.3\linewidth]{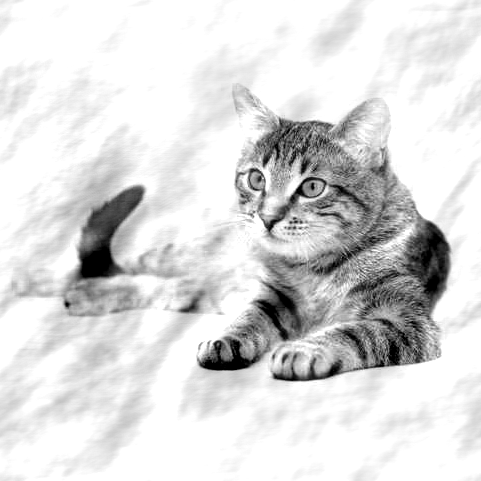}\\
(c) & (d)\\
\end{tabular}
\caption{ Demonstrating twicing in MR through a toy example \cite{CowtanCat}: 
given an unknown
image whose Fourier magnitudes are known through measurements, but phases are
missing, and a known similar image for which
both the Fourier magnitudes and phases are completely known.  (a) Original
image: unknown phases, known magnitudes (b) Similar image: known phases and
magnitudes (note that the tail is missing)
(c) Least squares estimator of original image, no magnitude correction (d)
Twicing for magnitude
correction. Note that the tail is better restored when twicing is used.}
\label{fig:tail_cat}
\end{figure}

As a natural generalization, one might wonder whether the estimator
$2\hat{\bA}_{\text{LS}}-\bB$ performs well for the non-scalar case $\bA\in\mathbb{R}^{N\times D}$ (or
$\mathbb{C}^{N\times D}$), where $(N,D)\neq (1,1)$.  In this paper, we consider 
the following problem: How to estimate
$\bA \in\mathbb{R}^{N\times D}$ (or $\mathbb{C}^{N\times D}$) from $\bC$ and
$\bB$, where $\bC=\bA \bA^*$ and $\bA=\bB+\bE$ for a matrix $\bE$ of small
magnitude? When  $N=D$, the result derived in this paper (see Sec. 
\ref{sec:estimator}) for an ``asymptotically consistent'' estimator of $\bA$ is
given by $\hat{\bA}_{\text{AT}} = \bB + \bU \bW {\bU}^{\star} 
(\hat{\bA}_{\text{LS}}-\bB)$
where $\bU$ and $\bW$ are defined in Sec. \ref{sec:estimator}. In particular, 
when $\bA \in\mathbb{C}^{1\times 1}$, this result coincides
with the result in ~\cite{Main1979} and justifies the approach of twicing. A 
formal proof
of the result derived in this paper is provided later in Sec. \ref{sec:proof}.

The motivation to study this problem is 3D structure determination 
in single particle reconstruction (SPR) without estimating the viewing angle associated with each image.
Although we focus on cryo-electron microscopy (cryo-EM) here, the methods in this paper can also be applied 
to SPR using X-ray free electron lasers (XFEL). In SPR using XFEL, short but 
intense pulses of X-rays are scattered from the molecule. The measured 2D 
diffraction patterns in random orientations are used to reconstruct the 3D 
diffraction volume by an iterative refinement procedure, akin to the approach in cryo-EM. 
Recently, there have been attempts to use Kam's theory for SPR using XFEL, to 
determine the 3D diffraction volume without any iterative refinement \cite{Saldin2009, Hosseinizadeh2015, Starodub12ncom}.

In this paper, we revisit Orthogonal Extension (OE) in
cryo-EM~\cite{Bhamre2014} that combines ideas from MR and 
Kam's autocorrelation analysis ~\cite{kam1980, Kam1985} for the purpose of 3D homology modeling, that is, for reconstruction of an unknown complex directly from its raw, noisy images when a previously solved similar complex exists. In SPR using cryo-EM \cite{resolution_revolution, Revol, 
Henderson}, the 3D structure of a macromolecule is reconstructed from its 
noisy, contrast transfer function (CTF) affected
2D projection images. Individual particle images are picked from micrographs, preprocessed, and used in further parts of the 
cryo-EM pipeline to obtain the 3D density map of the macromolecule. 
There exist many algorithms in popular cryo-EM software such as RELION, XMIPP, 
SPIDER, EMAN2, FREALIGN \cite{relion, xmipp, spider, eman2, frealign} that, 
given a starting 3D structure, refine it using the noisy 2D 
projection images. The result of the refinement procedure is often dependent on 
the choice of the initial model. It is therefore important to have a procedure 
to provide a good starting model for refinement. Also, a 
high quality starting model may significantly reduce the computational time 
associated with the refinement procedure (although we note recent advances in 
fast refinement \cite{Barnett2016,Punjani2017}). Such a high quality starting model can be obtained using OE. The main computational component of autocorrelation analysis is estimation of the covariance matrix of the 2D images. This computation requires only a single pass over the experimental images \cite{Zhao2016,Bhamre2016}. Autocorrelation analysis is therefore much faster than iterative refinement, which typically takes many iterations to converge. In fact,  the computational cost of autocorrelation analysis is even lower than that of a single refinement iteration, as the latter involves comparison of image pairs (noisy raw images with volume projections). OE can also be used for the purpose of model validation, being a complementary method for structure prediction. 

There are a few existing methods for ab-initio modeling. The random conical tilt method \cite{Radermacher2} can be used when 
two electron micrographs, one tilted and one untilted, are acquired with the 
same field of view. There are two main approaches for ab-initio estimation that 
do not involve tilting. One approach is to use the method of moments, that 
leverages the second order moments of the unknown 3D volume to estimate the 
particle orientations, but it suffers from being very sensitive to errors in 
the data \cite{Salzman1990, Goncharov1988}. The other approach is based on 
using common-lines between images \cite{VanHeel, Vainshtein, Singer2009, 
Singer2011}. However, common-lines based approaches have not been successful in 
obtaining 3D ab-initio models directly from raw, noisy images without performing 
any class averaging to suppress the noise. 

OE predicts the structure directly from the raw, noisy images without any averaging. The method is analogous to MR in X-ray crystallography for solving the missing phase problem. In OE, the homologous structure is used for estimating the missing orthogonal matrices associated with the spherical harmonics expansion of the 3D structure in reciprocal space. It is 
important to note that the missing orthogonal matrices in OE are not 
associated with the unknown pose of the particles, but with the
spherical harmonics expansion coefficients. The
missing coefficient matrices are, in general, rectangular of size $N \times D$, 
which serves as the
motivation to extend twicing to the general case of finding an estimator when
$(N,D) \neq (1,1)$. 

The paper is
organized as follows: First, we briefly review Kam's theory for autocorrelation 
analysis and describe the problem of OE in
cryo-EM in Sec. \ref{sec:oe}. Next, in Sec. \ref{sec:ls_est}, we describe 
the least squares solution to find an estimator to an unknown structure 
when we have
noisy projection images of the unknown structure,
and additional information about a homologous structure. 
In Sec. 4, we introduce Anisotropic Twicing as well as a family of estimators that interpolate between the least squares estimator and Anisotropic Twicing. We detail the procedure to estimate autocorrelation matrices and the algorithm of Orthogonal Extension with the Anisotropic Twicing estimator in Sec. 5. We benchmark the performance of these estimators through numerical experiments with synthetic 
and experimental datasets in Sec. 6. We provide a formal proof for asymptotic consistency 
of our Anisotropic Twicing correction scheme in
the general case of $(N,D)\neq (1,1)$ in the appendix (see Sec. 
\ref{sec:proof}). The code for all the algorithms in this paper is available in 
the open source software toolbox, ASPIRE, available for download at 
\url{spr.math.princeton.edu}. 

We apply anisotropic twicing to both synthetic and experimental cryo-EM 
datasets, and find that it recovers the unknown structure better than the least 
squares and twicing estimators on synthetic data. This is the first 
demonstration of reconstructing a starting 3D model in the presence of 
experimental conditions of CTF and noise without any class averaging, directly 
from raw images using the `Orthogonal Extension' procedure \cite{Bhamre2014}. 
While the anisotropic 
twicing estimator outperforms other estimators on synthetic datasets, in the case of 
the experimental dataset the reconstructions from all estimators are 
similar in quality, and any of these reconstructions can be used as a good starting point for refinement.

\section{Orthogonal Extension (OE) in Cryo-EM}
\label{sec:oe}
In ~\cite{Bhamre2014}, the authors presented two new approaches, collectively 
termed `Orthogonal Retrieval' methods, for 3D homology modeling based on Kam's theory ~\cite{kam1980}. Orthogonal Retrieval can be 
regarded as a generalization of the MR method from X-ray
crystallography to cryo-EM. 

Let $\Phi_A: \mathbb{R}^3 \rightarrow \mathbb{R}$ be the electron scattering 
density of the unknown structure, and let $\mathcal{F}(\Phi_A) : \mathbb{R}^3 
\to
\mathbb{C}$ be its 3D Fourier
transform. Consider the spherical harmonics expansion of $\mathcal{F}(\Phi_A)$
\begin{equation}
\label{eq:phia_expansion_full}
\mathcal{F}(\Phi_A)(k,\theta,\varphi) = \sum_{l=0}^{\infty} \sum_{m=-l}^{l} 
A_{lm}(k)
Y_l^m
(\theta, \varphi)
\end{equation}
where $k$ is the radial frequency and $Y_l^m$ are the real spherical
harmonics. Kam showed that the autocorrelation matrices 
\begin{equation}
\label{eq:Cl}
C_{{l}}(k_1,k_2) = \sum_{m=-l}^l A_{lm}(k_1)\overline{A_{lm}(k_2)}, \quad 
l=0,1,\ldots
\end{equation}
can be estimated from the covariance matrix of the 2D projection images whose 
viewing angles are uniformly distributed over the sphere.  This can be achieved with both clean as well as noisy images, as 
long as the number of noisy images is large enough to allow estimation of the 
underlying population covariance matrix of the clean images to the desired 
level of accuracy, using \cite{Bhamre2016}. The decomposition (3) suggests that the $l$'th order autocorrelation matrix $\bC_{\boldsymbol{l}}$ has a maximum rank of $2l+1$, and the maximum rank is even smaller in the presence of symmetry. 

While \eqref{eq:phia_expansion_full} is true if we 
want to represent the molecule to infinitely high resolution, in practice the 
images are sampled on a finite pixel grid and we cannot recover information beyond the Nyquist 
frequency. In addition, the molecule is compactly supported in $\mathbb{R}^3$, 
and the support size can also be estimated from the images. It is therefore 
natural to expand the volume in a truncated basis of spherical Bessel functions or 3D 
prolates. This leads to
\begin{equation}\label{eq:phia_expansion}
\mathcal{F}(\Phi_A)(k,\theta,\varphi) = \sum_{l=0}^{L} \sum_{m=-l}^{l} A_{lm}(k)
Y_l^m
(\theta, \varphi), \quad l=0, 1,\ldots, L 
\end{equation}
where the truncation $L$ is based on the resolution limit that can be achieved by the reconstruction. Our specific choice of $L$ is described after~\eqref{eq:nyq_sampling}. We can expand $A_{lm}(k)$ in a truncated basis of radial functions, chosen here as the spherical Bessel functions, as 
follows:
\begin{equation}
\label{Alm_practice}
A_{lm}(k)= \sum_{s=1}^{S_l} a _{lms}j_{ls}(k).
\end{equation}
Here the normalized spherical Bessel functions are
\begin{equation}
\label{j_ls}
j_{ls}(k)= \frac{1}{c\sqrt{\pi}|j_{l+1}(R_{l,s})|} j_l(R_{l,s}\frac{k}{c}), \quad 0<k<c, 
\quad s=1,2,\ldots,S_l,
\end{equation}
where $c$ is the bandlimit of the images, and $R_{l,s}$ is the $s$'th positive root of the equation $j_l(x)=0$. The functions
$j_{ls}$ are normalized such that
\begin{equation}
\int_0^c j_{ls}(k) j_{ls}^*(k) k^2 dk =1
\end{equation}
The number of radial basis functions $S_l$ in~\eqref{Alm_practice} is determined using the Nyquist criterion, similar to \cite{klug1972, Zhao2016}, where it has been described for 2D images expanded in a Fourier-Bessel basis (rather than 3D volumes as done here). We assume that the 2D images, and hence the 3D volume, are compactly supported on a disk of radius 
$R$ and have a bandlimit $0<c\leq0.5$. We require that the maximum of the inverse Fourier transform of the spherical Bessel function and its first zero after this maximum are both inside the sphere of compact support radius $R$. The truncation limit 
$S_l$ in~\eqref{Alm_practice} is then defined by the sampling criterion as the largest integer
$s$ that satisfies \cite{cheng2013random}
\begin{equation}
\label{eq:nyq_sampling}
R_{l,(s+1)}\leq 2\pi cR.
\end{equation}
$L$ in~\eqref{eq:phia_expansion} is the largest integer $l$ for which~\eqref{eq:nyq_sampling} has only one solution, that is, $S_l$ in~\eqref{Alm_practice} is at least $1$.
Each $\bC_{\boldsymbol{l}}$ is a matrix of size $S_l \times S_l$ when using the representation~\eqref{Alm_practice} in~\eqref{eq:Cl}. $S_l$ is a monotonically decreasing function of $l$ with approximately linear decay that we compute numerically. In
matrix notation,~\eqref{eq:Cl} can be written as
\begin{equation}
\label{eq:Cl_matrix}
\bC_{\boldsymbol{l}}=\bA_{\boldsymbol{l}} \bA_{\boldsymbol{l}}^*, 
\end{equation}
where $\bA_{\boldsymbol{l}}$ is a matrix of size
$S_l \times (2l+1)$, with $A_l(s,m) = a_{lms}$ in~\eqref{Alm_practice}. From \eqref{eq:Cl_matrix}, we note that $\bA_{\boldsymbol{l}}$ can be 
obtained from the Cholesky 
decomposition of $\bC_{\bl}$ up to a unitary matrix $\bU_{\boldsymbol{l}} \in \text{U}(2l+1)$ (the group of 
unitary matrices of size $(2l+1) \times (2l+1)$). Since $\Phi_A$ is real-valued, one can show using properties of its Fourier transform together with properties of the real spherical harmonics, that $A_{lm}(k)$ (and hence $\bA_{\boldsymbol{l}}$) is real for even $l$ and purely imaginary for odd $l$. So $\bA_{\boldsymbol{l}}$ is unique up to an orthogonal matrix $\bO_{\boldsymbol{l}} \in \text{O}(2l+1)$ (the group of 
orthogonal matrices of size $(2l+1) \times (2l+1)$). Determining $\bO_{\boldsymbol{l}}$ is the orthogonal retrieval problem in \cite{Bhamre2014}. 

If $S_l > 2l+1$, estimating the missing orthogonal matrix 
$\bO_{\boldsymbol{l}}$ is equivalent to estimating $\bA_{\boldsymbol{l}}$. Since $S_l$ is a decreasing function of $l$, for some large enough $l$ we would have $S_l < 2l+1$. For example, 
for the largest $l=L$ where $S_L=1$, $\bA_{\boldsymbol{L}}$ is of size $1 \times (2L+1)$, that 
is, it has $O(L)$ degrees of freedom. In such cases it does not make 
sense to estimate $\bO_{\boldsymbol{L}}$ which has $O(L^2)$ degrees of freedom. But we 
can still estimate $\bA_{\boldsymbol{L}}$ closest to $\bB_{\boldsymbol{L}}$ using \eqref{eq:leastsquareintro}.

\section{The Least Squares Estimator}
\label{sec:ls_est}
In this section we review the least squares estimator that was proposed in \cite{Bhamre2014}.
In order to determine the 3D Fourier transform $\mathcal{F}(\Phi_A)$ and thereby the 3D density  ${\Phi_A}$, we need to determine the coefficient matrices $\bA_{\boldsymbol{l}}$ of the spherical harmonic expansion. In OE, the coefficient matrices $\bA_{\boldsymbol{l}}$ are estimated with 
the aid of a homologous structure $\Phi_B$. Suppose $\Phi_B$ is a known 
homologous
structure, whose 3D Fourier transform $\mathcal{F}(\Phi_B)$ has the following 
spherical harmonic expansion:
\begin{equation}
\mathcal{F}(\Phi_B)(k,\theta,\varphi) = \sum_{l=0}^{\infty} \sum_{m=-l}^{l}
B_{lm}(k) Y_l^m (\theta, \varphi)
\end{equation} 
In practice, the homologous structure $\Phi_B$ is available at some finite resolution, therefore only a finite number of coefficient matrices $\bB_{\bl}$ ($l=0,1,\ldots,L_B$) are given. We show how to estimate the unknown structure $\Phi_A$ up to the resolution dictated by the input images and the resolution of the homologous structure through estimating the  coefficient matrices $\bA_{\bl}$ for $l=0,1,\ldots,L_A$ where $L_A = \min({L,L_B})$.

Let $\bF_{\boldsymbol{l}}$ be any matrix of size $S_l\times 2l+1$ satisfying $\bC_{\boldsymbol{l}}= \bF_{\boldsymbol{l}} \bF_{\boldsymbol{l}}^*$, determined from the
Cholesky decomposition of $\bC_{\boldsymbol{l}}$. Then, using \eqref{eq:Cl_matrix}
\begin{equation}
\bA_{\boldsymbol{l}} = \bF_{\boldsymbol{l}} \bO_{\boldsymbol{l}}
\end{equation}
where $\bO_{\boldsymbol{l}} \in \text{O}(2l+1)$ (for $S_l > 2l+1$). Using the assumption that 
the structures are homologous, $\bA_{\boldsymbol{l}} \approx \bB_{\boldsymbol{l}}$, one can determine $\bO_{\boldsymbol{l}}$ as 
the solution to the
least squares problem
\begin{equation}\label{ls}
\bO_{\boldsymbol{l}}=\argmin_{\bO \in \text{O}(2l+1)}  \|\bF_{\boldsymbol{l}} \bO - \bB_{\boldsymbol{l}}\|_F^2,
\end{equation}
where $\|\cdot\|_F$ denotes the Frobenius norm. 
Although the orthogonal group is non-convex, there is a closed form solution to
(\ref{ls}) (see, e.g., ~\cite{Keller1975}) given by
\begin{equation}
\bO_{\boldsymbol{l}}= \bV_{\boldsymbol{l}} \bU_{\boldsymbol{l}}^T,
\end{equation}
where 
\begin{equation}
\bB_{\boldsymbol{l}}^* \bF_{\boldsymbol{l}} = \bU_{\boldsymbol{l}} \boldsymbol{\Sigma}_{\boldsymbol{l}} \bV_{\boldsymbol{l}}^T
\end{equation}
is the singular value decomposition (SVD) of $\bB_{\boldsymbol{l}}^* \bF_{\boldsymbol{l}}$. Thus, $\bA_{\boldsymbol{l}}$ can be 
estimated
by the following least squares estimator:
\begin{equation}
 \hat{\bA}_{\boldsymbol{l},\text{LS}} = \bF_{\boldsymbol{l}} \bV_{\boldsymbol{l}} \bU_{\boldsymbol{l}}^T.
\end{equation} 
Hereafter, we drop the subscript $l$ for convenience, since the procedure can be 
applied to each $l$ separately.

\subsection{Algorithm 1: Orthogonal Extension by Least Squares}
\begin{algorithm}
\caption{Orthogonal Extension}
\begin{algorithmic}[1]
\label{algo:OE}
\Procedure{Orthogonal Extension by Least Squares (OE-LS): Estimate $\bA$ given 
$\bB \approx \bA$, subject to $\bC = \bA \bA^*$}{}\\
\Require $\bB \in \mathbb{C}^{N \times D}$, $\bC \in \mathbb{C}^{N \times N}$
\State Cholesky decomposition of $\bC$ to find an  $\bF \in \mathbb{C}^{N \times D}$ such that $\bC=\bF \bF^*$
\State Calculate $\bB^* \bF$ and its singular value decomposition 
$\bB^*\bF=\bU_0\bSigma_0\bV_0^*$.
\State The estimator is $\hat{\bA}_{\text{LS}}=\bF\bV_0\bU_0^*.$
\EndProcedure
\end{algorithmic}
\end{algorithm}

\section{Unbiased Estimator: Anisotropic Twicing}
\label{sec:estimator}
The case that $\bA$ is a complex-valued scalar, i.e., $\bA\in\mathbb{C}^{1\times
1}$ has been studied in X-ray crystallography. The   theoretical advantage of 
the unbiased estimator $2\hat{\bA}_{\text{LS}}-\bB$ for this case was elucidated in 
\cite{Main1979}. As a natural generalization, one may wonder whether the 
estimator
$2\hat{\bA}_{\text{LS}}-\bB$ is also unbiased for $(N,D)\neq (1,1)$. We assume that $\bA$ 
is sampled from the model $\bA=\bF\bV$, where $\bF\bF^*=\bC$ and $\bV$ is a 
random orthogonal matrix (or a random unitary matrix) sampled from the uniform 
distribution with Haar measure over the orthogonal group when $\bA$ is a 
real-valued matrix or the unitary group (when $\bA$ is a complex-valued 
matrix). This probabilistic model is reasonable for \eqref{eq:leastsquareintro}, 
because when $\bA\bA^\star$ is given, $\bF$ is known and $\bV$ is an unknown 
orthogonal or unitary matrix, that is, we have no prior information about $\bV$. 
In addition, we assume that $\bB$ is a matrix close to $\bA$ such that $\bA-\bB$ 
is fixed. Our goal is to find an unbiased estimator of $\bA$ which is an affine 
transformation of $\hat{\bA}_{\text{LS}}$. The main result is as follows:

\begin{thm}\label{thm:main}
When  $N = D$, assuming that the spectral decomposition of $\bC$ is given by 
$\bC=\bU\diag(\lambda_1,\lambda_2,\cdots,\lambda_D)\bU^*$, then using our 
probabilistic model we have \begin{equation}\label{eq:expectation}
\Expect[\bA-\hat{\bA}_{\text{LS}}]=\bU\bT\bU^*(\bA-\bB) + o(\|\bA-\bB\|_F),
\end{equation}
where $\bT$ is a diagonal matrix with $i$-th diagonal entry given by 
\[\bT_{ii}=\begin{cases}\frac{1}{D}\Big[-\frac{1}{2}+\sum_{1\leq j\leq 
D}\frac{\lambda_i^2}{\lambda_i^2+\lambda_j^2}\Big]\,\,\,\text{when 
$\bA,\bC\in\mathbb{R}^{D\times D}$},\\\frac{1}{D}\sum_{1\leq j\leq 
D}\frac{\lambda_i^2}{\lambda_i^2+\lambda_j^2}
\,\,\,\text{when $\bA,\bC\in\mathbb{C}^{D\times D}$},
\end{cases}\] 
and $f(\bX)=o(\|\bX\|_F)$ means that $\limsup_{\|\bX\|_F\rightarrow
0}{f(\bX)/\|\bX\|_F}\rightarrow 0$.
\end{thm}
 
From \eqref{eq:expectation}, we have 
\[
(-\bI+\bU\bT\bU^*)(\bA-\bB)=\bB - \Expect[\hat{\bA}_{\text{LS}}] + o(\|\bA-\bB\|_F) 
\]
and an ``asymptotically consistent'' estimator of $\bA$ is given by
\begin{equation}\label{eq:solution}
\hat{\bA}_{\text{AT}}=\bB-(\bI-\bU\bT\bU^*)^{-1}(\bB-\hat{\bA}_{\text{LS}}) =\bB+\bU\bW\bU^*(\hat{\bA}_{\text{LS}}-\bB),
\end{equation}
 where $\bW=(\bI-\bT)^{-1}$. 
 
A formal proof
of Theorem~\ref{thm:main} is provided in the appendix (Sec.~\ref{sec:proof}). In particular, when $\bA,\bB\in\mathbb{C}^{1\times 1}$, the matrices reduce to 
scalars: $\bU=1$, $\bT=\frac{1}{2}$, $\bW=2$ and 
$\hat{\bA}_{\text{AT}}=\bB+2(\hat{\bA}_{\text{LS}}-\bB)$. This result coincides
with the result in~\cite{Main1979} and justifies the approach of ``twicing''.

\subsection{A family of estimators}
\label{sec:fam_estimators}
From~\eqref{eq:expectation} it follows that
 \begin{equation}\label{eq:temp}
\Expect[\bA]=\Expect[\hat{\bA}_{{\text{LS}}}]+\bU\bT\bU^*(\bA-\bB) + 
o(\|\bA-\bB\|_F).
\end{equation}
Following the spirit of Tukey's twicing, we could approximate $\bA$ in the RHS of 
\eqref{eq:temp} by $\hat{\bA}_{\text{LS}}$, which leads to a new estimator
\[
\hat{\bA}_{\text{T}}^{(1)}=\hat{\bA}_{\text{LS}}+\bU\bT\bU^*(\hat{\bA}_{\text{LS}}
-\bB).
\]
In fact, there exists a family of estimators by approximating $\bA$ recursively in the RHS 
of \eqref{eq:temp} by $\hat{\bA}_{\text{T}}^{(t-1)}$ (with $\hat{\bA}_{\text{T}}^{(0)} = \hat{\bA}_{\text{LS}}$):
\begin{equation}
\hat{\bA}_{\text{T}}^{(t)}=\hat{\bA}_{\text{LS}}+\bU\bT\bU^*(\hat{\bA}_{\text{T}}^{(t-1)}
-\bB).
\end{equation}
This family of estimators can be explicitly written as
\begin{equation}
\label{eqn:family_est}
\hat{\bA}_{\text{T}}^{(t)}=\bB+\bU(\bI+\bT+\bT^2+\cdots+\bT^t)\bU^*(\hat{\bA}_{\text{LS}}-\bB).
\end{equation}
Using $\bW=(\bI-\bT)^{-1}=\sum_{i=0}^\infty\bT^i$, we have that $\bA_{\text{T}}^{(t)}\rightarrow 
\hat{\bA}_{\text{AT}}$ as $t\rightarrow\infty$.

In general, this family of estimators has smaller variance than 
$\hat{\bA}_{\text{AT}}$, but larger bias since they are not unbiased (see Fig.~\ref{fig:biasvar}).

\subsection{Generalization to the setting $N\neq D$}

If $N > D$, then the column space of $\bA$ is the same as the column space of 
$\bC$. Let $\bP$ be the projector of size $N\times D$ to this column space, 
then we have $\bA=\bP\bP^*\bA$. As a result, to find an unbiased estimator of 
$\bA$, it is sufficient to find an unbiased estimator of $\bP^*\bA$, which is a 
square matrix. Since $\bP^*\bA$ is close to $\bP^*\bB$ and 
$(\bP^*\bA)(\bP^*\bA)^*=\bP^*\bC\bP^*$ is known, Theorem~\ref{thm:main} is 
applicable, and an unbiased estimator of $\bP^*\bA$ can be obtained through 
\eqref{eq:solution}, with $\bB$ replaced by $\bP^*\bB$ and $\bC$ replaced by 
$\bP^*\bC\bP^*$. In summary, an unbiased estimator of $\bA$ can be obtained in 
two steps:
\begin{enumerate}
\item Find $\hat{\bA}^{(0)}_{\text{AT}}$, an unbiased estimator of $\bP^*\bA$, 
by applying \eqref{eq:solution}, with $\bB$ replaced by $\bP^*\bB$ and $\bC$ 
replaced by $\bP^*\bC\bP^*$.
\item An unbiased estimator of $\bA$ is obtained by 
$\hat{\bA}_{\text{AT}}=\bP\hat{\bA}^{(0)}_{\text{AT}}$.
\end{enumerate}

If $N<D$, we use the following heuristic estimator. Let $\bP$ be a matrix of 
size $D\times N$ that is the projector to the row space of $\bB$, and assuming 
that $\hat{\bA}$, the estimator of $\bA$, has the same row space as $\bB$, then 
$\hat{\bA}=\hat{\bA}\bP\bP^*$, and it is sufficient to find $\hat{\bA}\bP$, an 
estimator of $\bA\bP$. With $(\bA\bP)(\bA\bP)^*=\bA\bA^*=\bC$ known and the 
fact that $\bA\bP$ is close to $\bB\bP$, we may use the estimator 
\eqref{eq:solution}. In summary, we use the following procedure:
\begin{enumerate}
\item Find $\hat{\bA}^{(0)}_{\text{AT}}$, an estimator of $\bA\bP$, by applying 
the estimator \eqref{eq:solution}, with $\bB$ replaced by $\bB\bP$.
\item An estimator of $\bA$ is obtained by 
$\hat{\bA}_{\text{AT}}=\hat{\bA}^{(0)}_{\text{AT}}\bP^*$.
\end{enumerate}
We remark that for $N<D$ there is no theoretical guarantee to show that it is an unbiased 
estimator, unlike the setting $N\geq D$. However, the assumption that 
$\hat{\bA}$ has the same column space as $\bB$ is reasonable, and the proposed 
estimator performs  well in practice.

\section{Estimation of the Covariance and Autocorrelation Matrices}
\label{Sec:oeat_algo}
The autocorrelation matrices $\bC_{\bl}$ in Kam's theory are derived from the covariance matrix $\boldsymbol{\Sigma}$ of the 2D Fourier transformed projection images through \cite{kam1980}
\begin{equation}\label{eq:kam_cl}
{\bC}_{\bl} (|k_1|,|k_2|) = 2\pi(2l+1) \int_0^\pi {{\Sigma}}(|k_1|,|k_2|,\psi) P_l(\cos{\psi}) \sin{\psi} d\psi
\end{equation}
where $\psi$ is the angle between the vectors $k_1$ and $k_2$ in the x-y plane.
We estimate the covariance matrix $\boldsymbol{\Sigma}$ of 
the underlying 2D Fourier transformed clean projection images using the method described in
\cite{Bhamre2016}. This estimation method provides a more accurate covariance compared to the classical sample covariance matrix \cite{vanHeel1981, vanHeel1984}. First, it corrects for the CTF. Second, it performs eigenvalue shrinkage, which is critical for high dimensional statistical estimation problems. Third, it exploits the block diagonal structure of the covariance matrix in a steerable basis, a property that follows from the fact that any experimental image is just as likely to appear in different in-plane rotations. A steerable basis consists of outer products of radial functions (such as Bessel functions) and Fourier angular modes. Each block along the diagonal corresponds to a different angular frequency \cite{Zhao1}. Moreover, the special block diagonal structure facilitates fast computation of the covariance matrix \cite{Zhao2016}.

Since the autocorrelation matrix $\bC_{\bl}$ estimated from projection images can have a rank exceeding $2l+1$, we first find its best rank $2l+1$ approximation via singular value decomposition, before computing its Cholesky decomposition. In the case of symmetric molecules, we use the appropriate rank as dictated by classical representation theory of $SO(3)$ \cite{Klein1914, cheng2013random} (less than $2l+1$).

\subsection{Algorithm 2: Orthogonal Extension by Anisotropic Twicing}
\begin{algorithm}
\caption{Orthogonal Extension}
\begin{algorithmic}[1]
\label{algo:OEAT}
\Procedure{Orthogonal Extension by Anisotropic Twicing (OE-AT): Estimate $\bA$ 
given $\bB \approx \bA$, subject to $\bC = \bA \bA^*$}{}\\
\Require $\bB \in \mathbb{C}^{N \times D}$, $\bC \in \mathbb{C}^{N \times N}$
\State Find any $\bF \in \mathbb{C}^{N \times D}$ such that $\bC=\bF \bF^*$
\State Calculate $\bB^* \bF$ and calculate its singular value decomposition 
$\bB^*\bF=\bU_0\bSigma_0\bV_0^*$.
\State Calculate the OE-LS estimator is  
$\hat{\bA}_{\text{LS}}=\bF\bV_0\bU_0^*$, (see Algorithm 1).
\State For $N=D$, the OE-AT estimator is given by 
$\hat{\bA}_{\text{AT}}=\bB+\bU\bW\bU^*(\hat{\bA_{\text{LS}}}-\bB)$.
\State For $N>D$, assuming that $\bP$ is the projector of size $N\times D$ to 
the $D$-dimensional subspace spanned by the columns of $\bC$, 
$\hat{\bA}_{\text{AT}}=\bP\hat{\bA}_{\text{AT}}^{(0)}.$
\State For $N<D$, assuming that $\bP$ is the projector of size $D\times N$ to 
the $N$-dimensional subspace in $\reals^D$ spanned by the rows of $\bB$, 
$\hat{\bA}_{\text{AT}}=\hat{\bA}_{\text{AT}}^{(0)}\bP^*.$
\EndProcedure
\end{algorithmic}
\end{algorithm}

\section{Numerical Experiments}
\subsection{Bias Variance Trade-off}
\label{subsec:biasvar}
For any parameter $\theta$, the performance of its estimator $\hat{\theta}$ can 
be measured in terms of its mean squared error (MSE), 
$\mathbb{E}||\theta-\hat{\theta}||^2$.
The MSE of any estimator can be decomposed into its bias and variance:
\begin{equation}
\label{eqn:biasvar}
\text{MSE}=\mathbb{E}[||\theta-\hat{\theta}||^2] = ||\text{Bias}||^2 + \text{Var}
\end{equation}
where
\begin{equation}
\text{Bias}= \mathbb{E}[\hat{\theta}]-\theta 
\end{equation}
and
\begin{equation}
\text{Var}= \mathbb{E}[||\hat{\theta}-\mathbb{E}[\hat{\theta}] ||^2]
\end{equation}
Unbiased estimators are often not optimal in terms of 
MSE, but they can be valuable for being unbiased. We performed a numerical 
experiment starting with a fixed $\bF \in \mathbb{R}^{10 \times 10}$ and an 
unknown matrix $\bA = \bF \bO$ where $\bO$ is a random orthogonal matrix. We 
are given a known similar matrix $\bB$ such that $\bA = \bB + \bE$. The goal is 
estimate $\bA$ given $\bB$ and $\bF$.  Figure \ref{fig:biasvar} shows a 
comparison of the bias and root mean squared error (RMSE) of different estimators averaged over $10000$ runs of 
the numerical experiment: the anisotropic twicing estimator, the twicing 
estimator, the least squares estimator, and estimators from the family of 
estimators for some values of $t$ in~Sec. \ref{sec:fam_estimators}. The figure demonstrates that the AT estimator 
is asymptotically unbiased at the cost of higher MSE. 
\begin{figure}[]\label{fig:biasvar}
\includegraphics[width=0.9\linewidth]{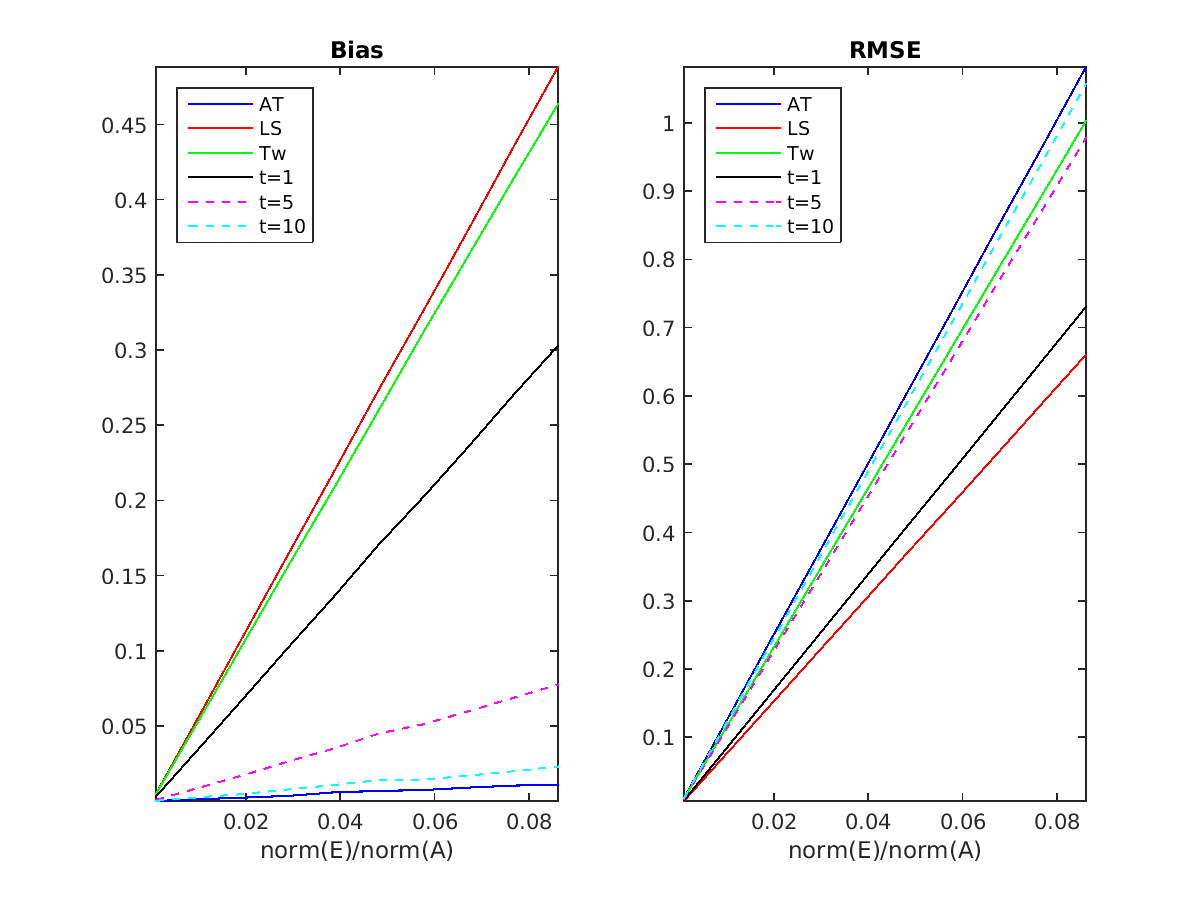}%
\caption{Bias and RMSE of the Anisotropic Twicing (AT), Least Squares (LS), 
Twicing (Tw) estimators and also the family of estimators with $t=1,5,10$  averaged over $10000$ experiments, as described in 
Sec.~\ref{subsec:biasvar}. The x-axis shows the relative perturbation 
$||\bE||/||\bA||$. }
\end{figure}

\subsection{Synthetic Dataset: Toy Molecule}
We perform numerical experiments with a synthetic dataset generated from an 
artificial `Mickey Mouse' molecule. The molecule $\bB$ is made up of 
ellipsoids, and the density is set to 1 inside the ellipsoids and to 0 outside. 
Fig. \ref{fig:mickey_recon} shows the artificial new volume $\bA = \bB + \bE$ 
created by adding a small ellipsoid $\bE$, which we will refer to here as the 
``nose'', to the original mickey mouse volume $\bB$. This represents the small 
perturbation $\bE$. When the Fourier volume $\bB + \bE$ is expanded in the 
truncated spherical Bessel basis described in Sec.~\ref{Sec:oeat_algo}, the 
average relative perturbation $||\bE_{\boldsymbol{l}}||/||\bA_{\boldsymbol{l}}||$ for the first few coefficients 
in the truncated spherical harmonic expansion for $l=1,\ldots,10$ is $8\%$.

Next, we generate $10000$ projection images from the volume $\bA$. We then 
employ OE to reconstruct the volume $\bA$ from $\bB$ and the clean projection 
images of $\bA$. Fig. \ref{fig:mickey_recon} shows the reconstructions obtained 
using each of the three estimators in the OE framework, visualized in Chimera \cite{chimera}. 
We note that while all three estimators are able to recover the additional 
subunit $\bE$, the AT estimator best recovers the unknown subunit to its 
correct relative magnitude. The relative error in the region of the unknown 
subunit $\bE$ is $59\%$ with least squares, $31\%$ with twicing and $19\%$ with 
anisotropic twicing.

\begin{figure}[]
\caption{A synthetic toy mickey mouse molecule with a small additional subunit, 
marked `E' in (a). We reconstruct the molecule $\bA$ from its clean projection 
images, given $\bB$. We show reconstructions obtained with the least squares 
estimator in (b), twicing estimator in (c), and AT estimator in (d). }
\includegraphics[width=0.95\linewidth]{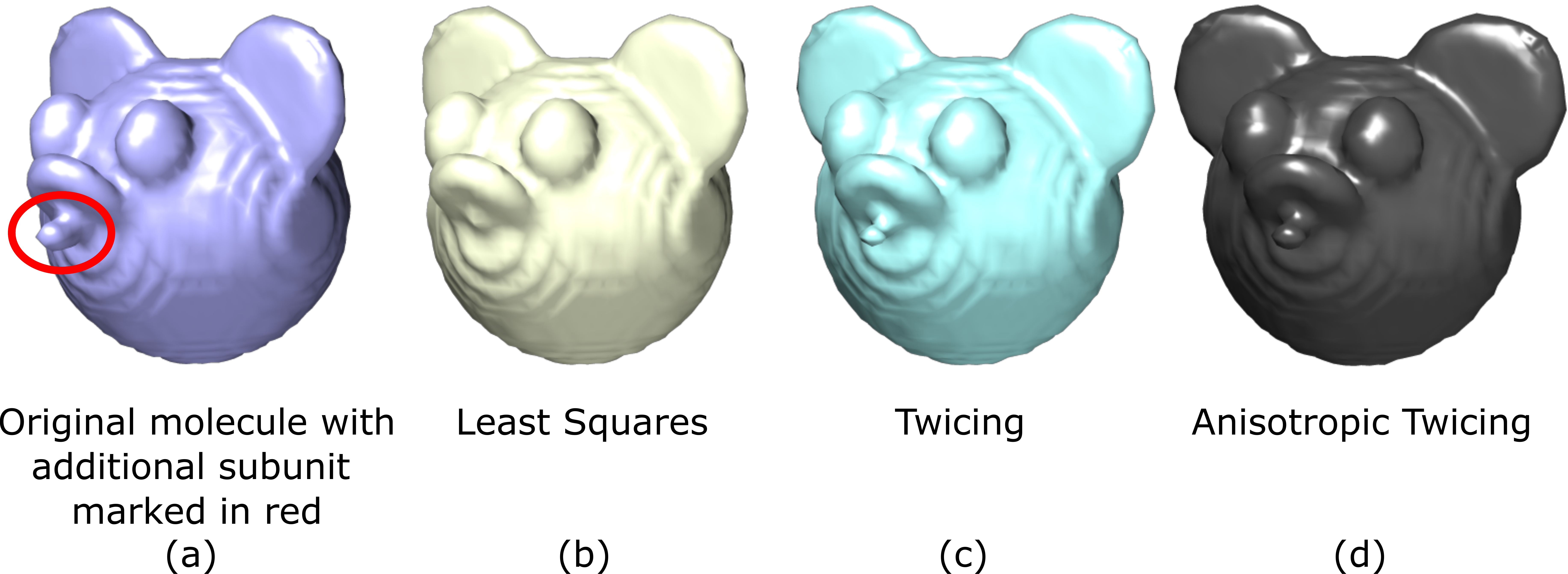}
\label{fig:mickey_recon}
\end{figure}

\subsection{Synthetic Dataset: TRPV1}

We perform numerical experiments with a synthetic dataset generated from the 
TRPV1 molecule (with imposed $C_4$ rotational symmetry) in complex with DkTx and RTX ($\bA$). This volume is available on EMDB 
as EMDB-8117. The small additional subunit is visible as an extension over the 
top of the molecule, shown in Fig. \ref{fig:simtrpv}(i). This represents the 
small perturbation $\bE$. 

Next, we generate $26000$ projection images from $\bA$, add 
the effect of both the CTF (the images are divided into $10$ defocus groups) and additive white Gaussian noise (SNR=$1/40$) and use OE 
to reconstruct the volume $\bA$. Fig. \ref{fig:simtrpv}(ii) 
shows the reconstructions obtained using each of the three estimators in the OE 
framework, visualized in Chimera. The $C_4$ symmetry was taken into account in the autocorrelation analysis by including in (2) only symmetry-invariant spherical harmonics $Y_l^m$ for which $m=0\mod4$. As seen earlier with the synthetic case, all 
three estimators are able to recover the additional subunit $\bE$, while the AT 
estimator best recovers the unknown subunit to its correct relative magnitude. 
The relative error in the unknown subunit is $43\%$ with least squares, 
$56\%$ with twicing and $30\%$ with anisotropic twicing. We note that 
this is the first successful attempt, even with synthetic data, at using OE for 
3D homology modeling directly from CTF-affected and noisy images (at experimentally relevant conditions). The numerical 
experiments using the Kv1.2 potassium channel in \cite{Bhamre2014} were at an unrealistically high SNR and did not include the effect of the CTF. The 
reason for this improvement is the improved covariance estimation 
\cite{Bhamre2016}.

\begin{figure}[]
\label{fig:simtrpv}
\caption{A synthetic TRPV1 molecule (EMDB 8118), with a small additional 
subunit DxTx and RTX (EMDB 8117), marked `E' in (i-b). We reconstruct the molecule from its noisy, CTF-affected images, and the homologous structure. In (ii), we show 
reconstructions obtained with the least squares, twicing and AT estimators 
using OE, along with the homologous structure and the ground truth projected on to the basis in  (ii-a) and (ii-b).}
\centering
\begin{tabular}{cc}
\includegraphics[width=0.45\linewidth]{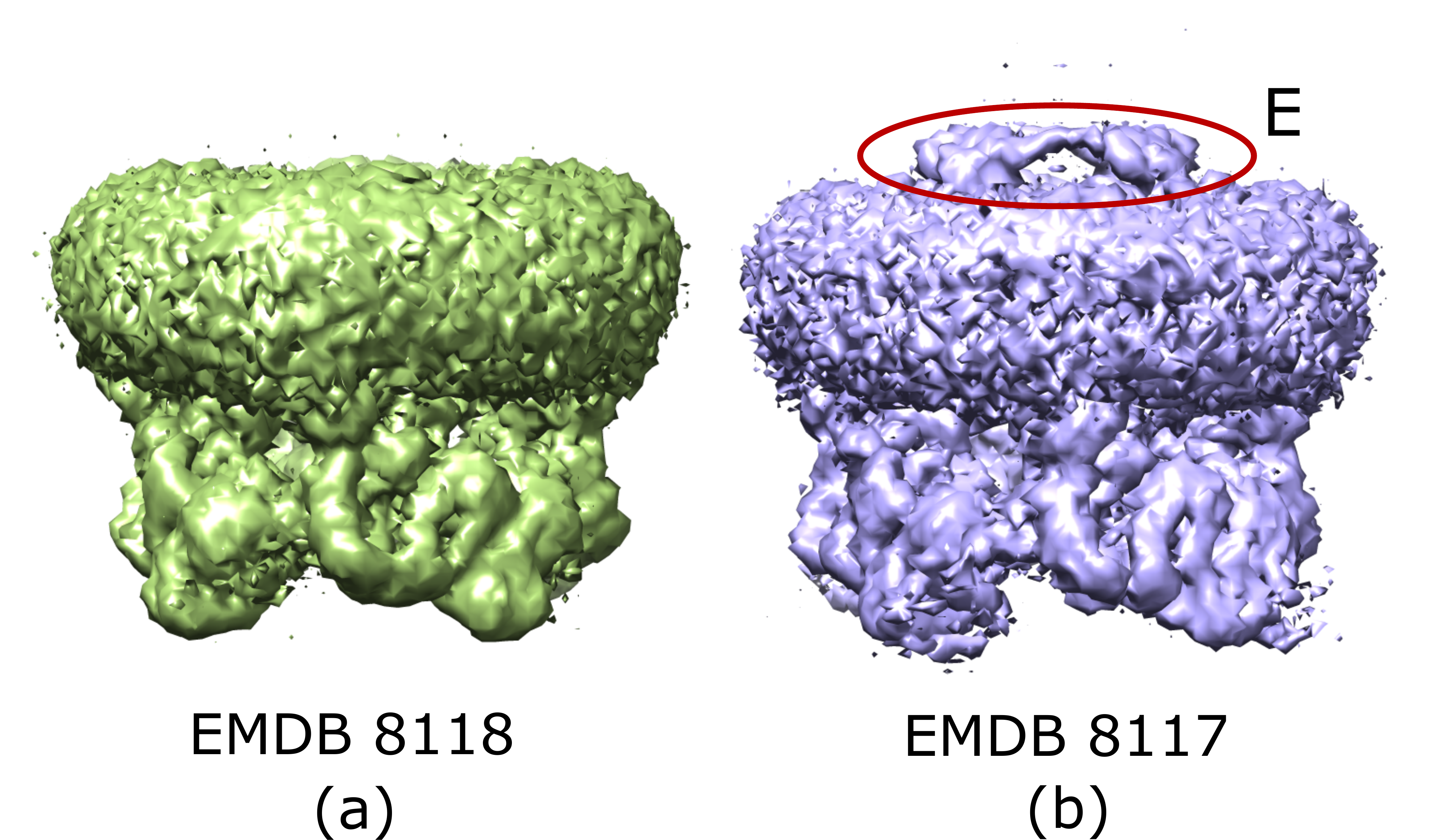}\label{fig:simtrpv_emdb} \\
(i) \\ 
\includegraphics[width=0.95\linewidth]{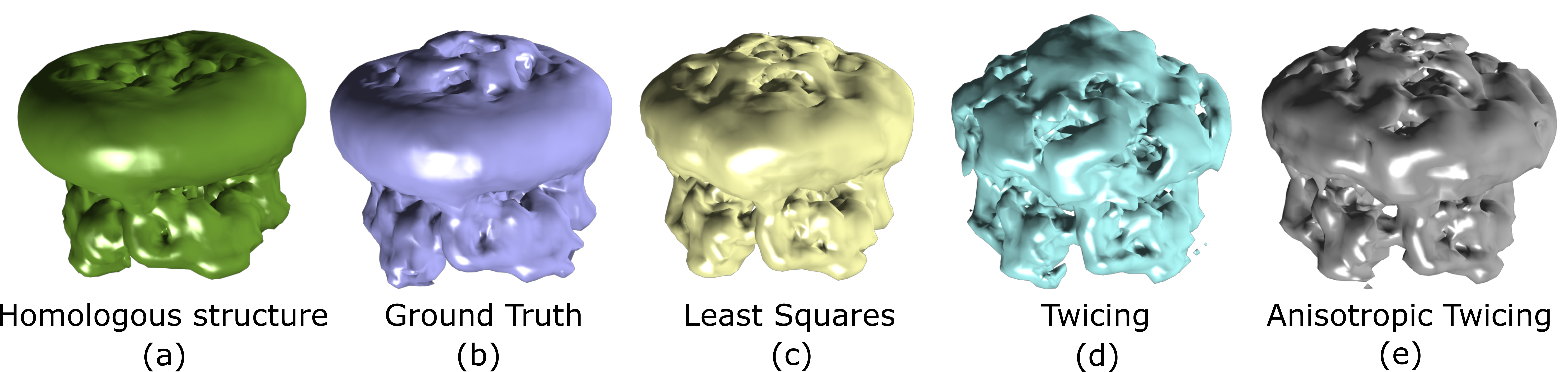}\label{fig:simtrpv_res}\\
(ii)\\
\end{tabular}

\end{figure}

\subsection{Experimental Dataset: TRPV1}

We apply OE to an experimental data of the TRPV1 molecule in complex with DkTx and RTX, 
determined in lipid nanodisc, available on the public database Electron 
Microscopy Pilot Image Archive (EMPIAR) as EMPIAR-10059, and the 3D 
reconstruction is available on the electron microscopy data bank (EMDB) as EMDB-8117, courtesy of Y. Gao et al 
\cite{gao16}. The dataset provided consists of $73000$ motion corrected, picked 
particle images (which were used for the reconstruction in EMDB-8117) of size $192 \times 192$ with a pixel size $1.3\AA$. We use the 3D 
structure of TRPV1 alone as the similar molecule. This is available on the EMDB as 
EMDB-8118. The two structures differ only by the small DkTx and RTX subunit at 
the top, which can be seen in Fig.~\ref{fig:simtrpv}(i).

Since the noise in experimental images is colored while our covariance estimation procedure requires white noise, we first preprocess the raw images in order to ``whiten" 
the noise. We estimate the power spectrum of noise using the corner pixels of 
all images. The images are then whitened using the estimated noise power 
spectrum.

In the context of our mathematical model, the volume EMDB-8117 of TRPV1 with DkTx 
and RTX is the unknown volume $\bA$, and the volume EMDB-8118 of TRPV1 
alone is the known, similar volume $\bB$. We use OE to estimate $\bA$ 
given $\bB$ and the raw, noisy projection images of $\bA$ from an 
experimental dataset. 

The basis assumption in Kam's theory is that the distribution of viewing angles is uniform. This assumption is difficult to satisfy in
practice, since molecules in the sample can often have preference for certain orientations
due to their shape and mass distribution. The viewing angle distribution in EMPIAR-10059
is non-uniform (see Fig.~\ref{Fig:viewing_angles}). As a robustness test of our methods, we attempt
3D reconstruction with (i) all images, such that the viewing angle distribution is non-uniform, as
well as (ii) by sampling images such that the viewing angle distribution of the images is 
approximately uniform (as shown in Fig.~\ref{Fig:viewing_angles}). We obtained the
final viewing angles estimated after refinement from the Cheng lab at UCSF \cite{gao16}. Our
sampling procedure is as follows: we choose $10000$ points at random from the uniform distribution on the sphere and classify each image into
these $10000$ bins based on the point closest to it. We discard bins that have no images, and for the remaining
bins we pick a maximum of $3$ points per bin. We use the selected images (slighty less less than $30000$) for reconstruction with
roughly uniform distributed viewing angles.

The reconstructed 3D volumes are 
shown in Fig.~\ref{fig:realtrpv}. We note that the additional subunit is 
recovered at the right location, and roughly to the expected size, 
using all three estimators. This is the first instance of reconstructing a 3D 
model directly from raw experimental images, without any class
averaging or iterative refinement, by employing OE. The Fourier cross resolution (FCR) of the reconstruction with the `ground truth' EMDB-8117 is shown in Fig.~\ref{Fig:fsc}.

The algorithm is
implemented in the UNIX environment, on a machine with 60 cores,
running at 2.3 GHz, with total RAM of 1.5TB. Using 20 cores, the total time taken here for preprocessing (whitening, background normalization. etc.) the 2D images and computing the covariance matrix was 1400 seconds. Calculating the autocorrelation matrices using~\eqref{eq:kam_cl} involves some numerical integration (eq. 7.15 in \cite{cheng2013random}) which took 790 seconds, but for a fixed $c$ and $R$ (satisfied for datasets of roughly similar size and quality) these can be precomputed. Computing the basis functions and calculating the coefficient matrices $\bA_{\bl}$  of the homologous structure took 30 seconds and recovering the 3D structure by applying the appropriate estimator (AT, twicing, or LS) and computing the volume from the estimated coefficients took 10 seconds.

\begin{figure}[]
\label{fig:realtrpv}
\caption{OE with an experimental data of the TRPV1 in complex with DkTx and RTX (EMPIAR-10059) whose 3D 
reconstruction is available as EMDB-8117. 3D reconstructions with OE using the least squares, anisotropic twicing, and twicing estimators: (i) With (slightly less than $30000$) images selected by sampling to impose approximately uniform viewing angle distribution (ii) With all $73000$ images such that the viewing angle distribution is non-uniform (see Fig.~\ref{Fig:viewing_angles}).}
\centering
\begin{tabular}{cc}
\includegraphics[width=0.9\linewidth]{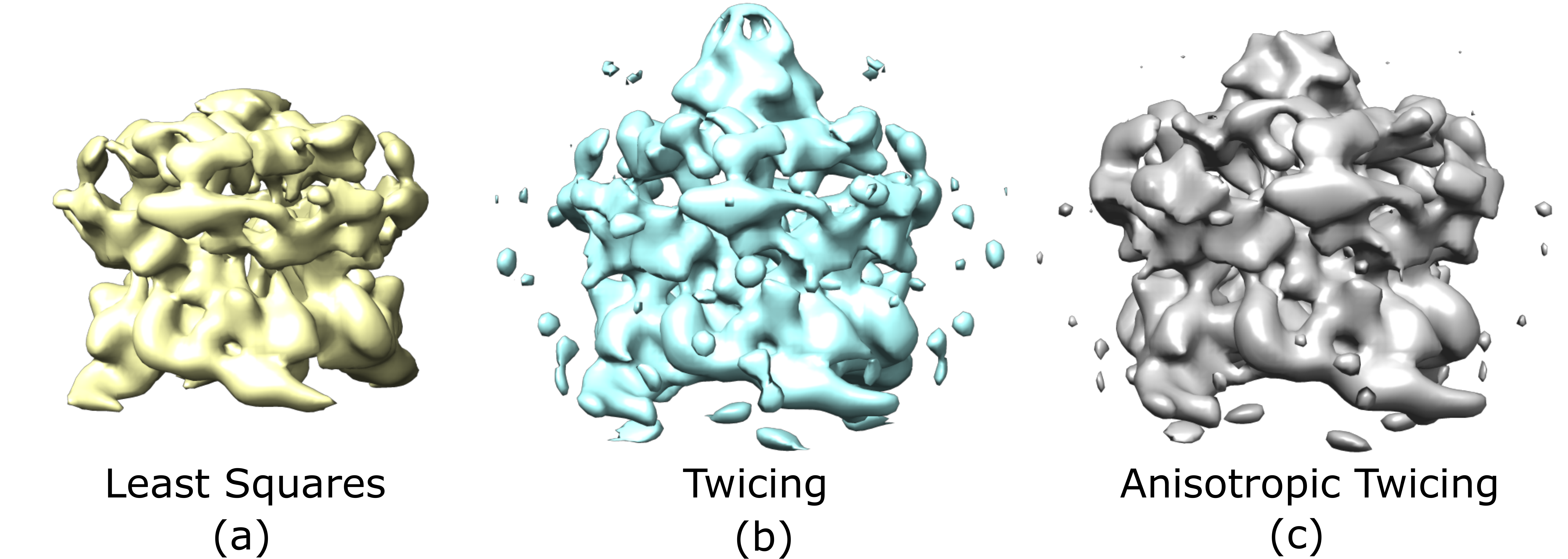} \\
(i) \\ 
\includegraphics[width=0.9\linewidth]{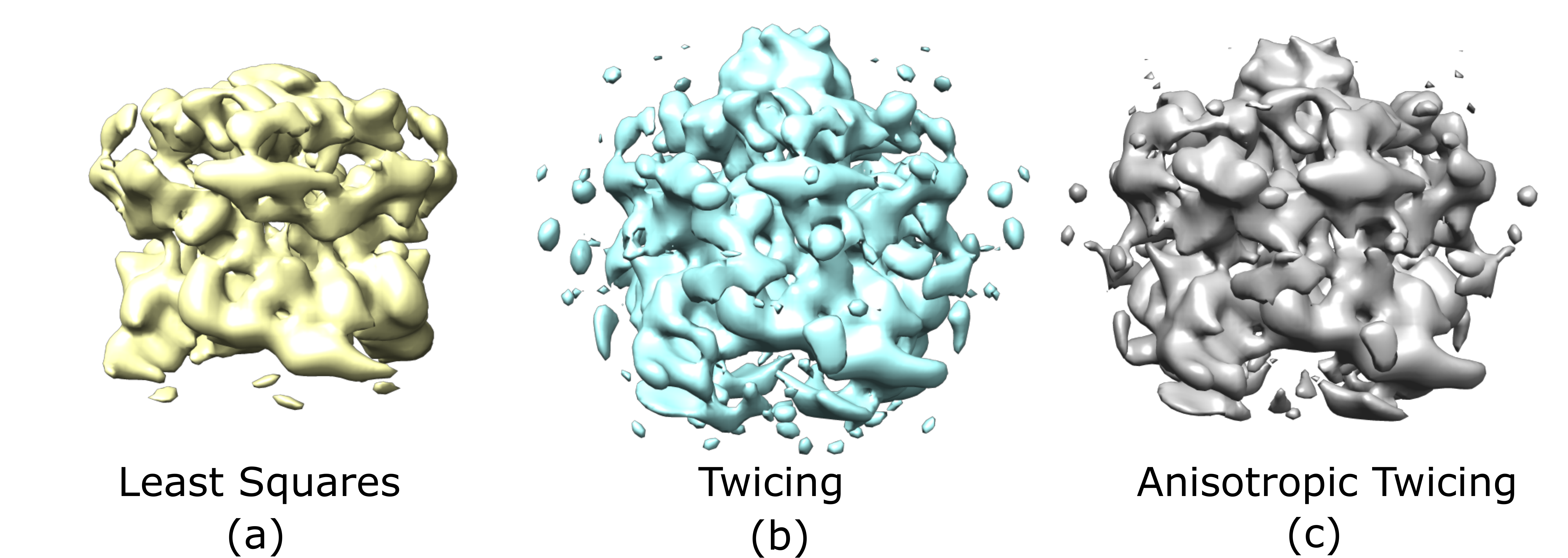}\\
(ii)\\
\end{tabular}
\end{figure}

\begin{figure}[]
\caption{Viewing angle distribution of images in the dataset EMPIAR-10059: (i) Non-uniform distribution in the raw dataset. The visualization here shows centroids of the bins that the sphere is divided into. The color of each point is assigned based on the number of points in the bin, yellow being the largest, representing the most dense bin, and blue being the smallest. (ii) Approximately uniform distribution after sampling.}
\centering
\begin{tabular}{cc}
\includegraphics[width=0.47\textwidth]{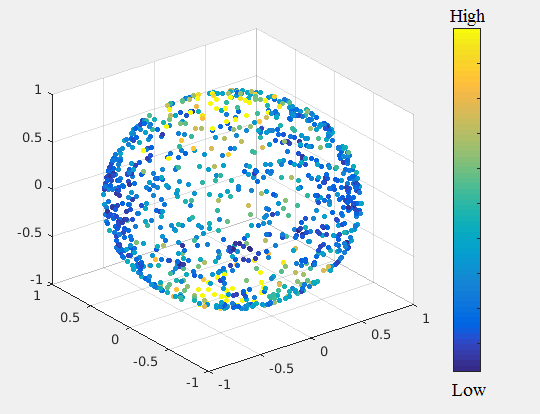}
&\includegraphics[width=0.47\textwidth]{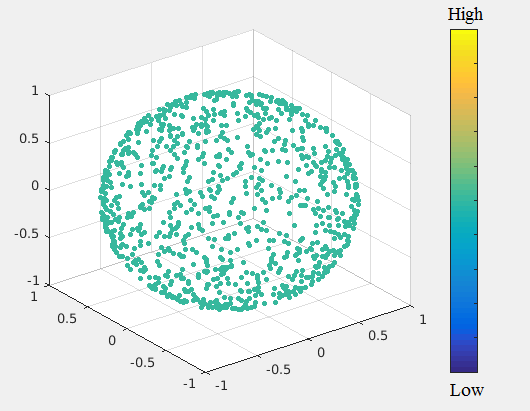}\\
(i) &(ii)\\
\end{tabular}
\label{Fig:viewing_angles}
\end{figure}


\begin{figure}[]
\caption{(i) FCR curve for the reconstruction of the entire molecule obtained by OE using the least squares, twicing, anisotropic twicing estimators corresponding to Fig.~\ref{fig:realtrpv}(i). (ii) FCR curve for the reconstruction of the unknown subunit obtained by OE using the least squares, twicing, anisotropic twicing estimators corresponding to Fig.~\ref{fig:realtrpv}(i). We also show the FCR of the masked homologous volume (EMDB-8118) to show the improvement in FCR obtained using OE.}
\centering
\begin{tabular}{cc}
\includegraphics[width=0.47\textwidth]{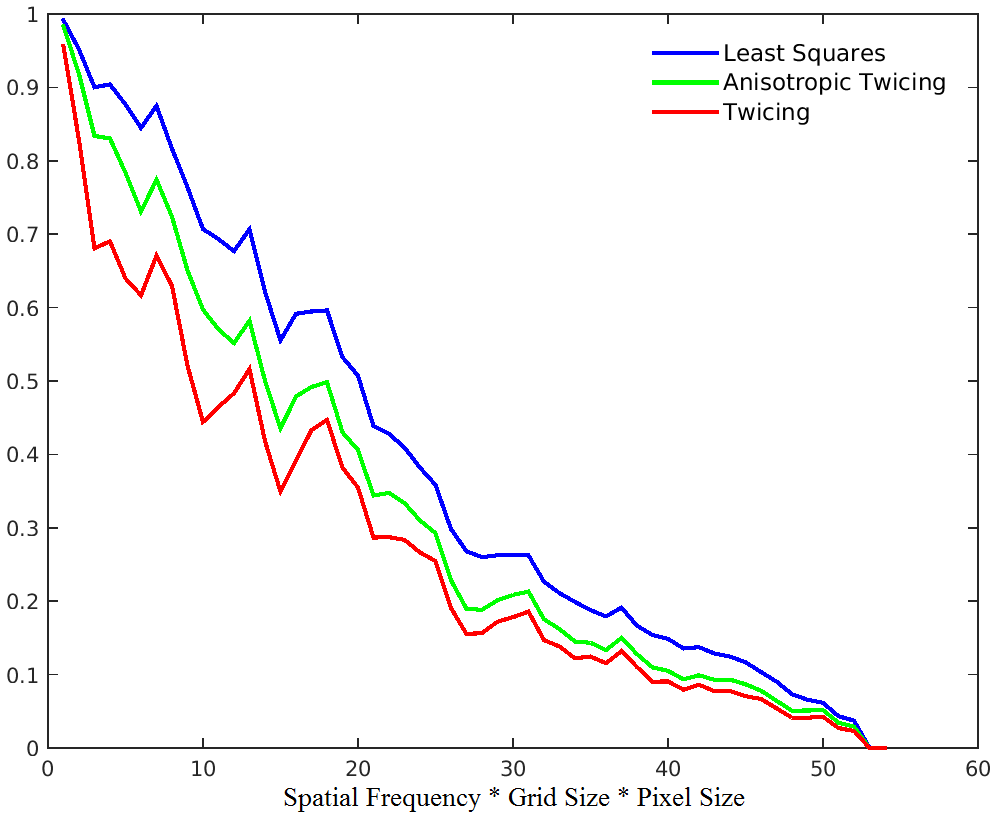}
&\includegraphics[width=0.47\textwidth]{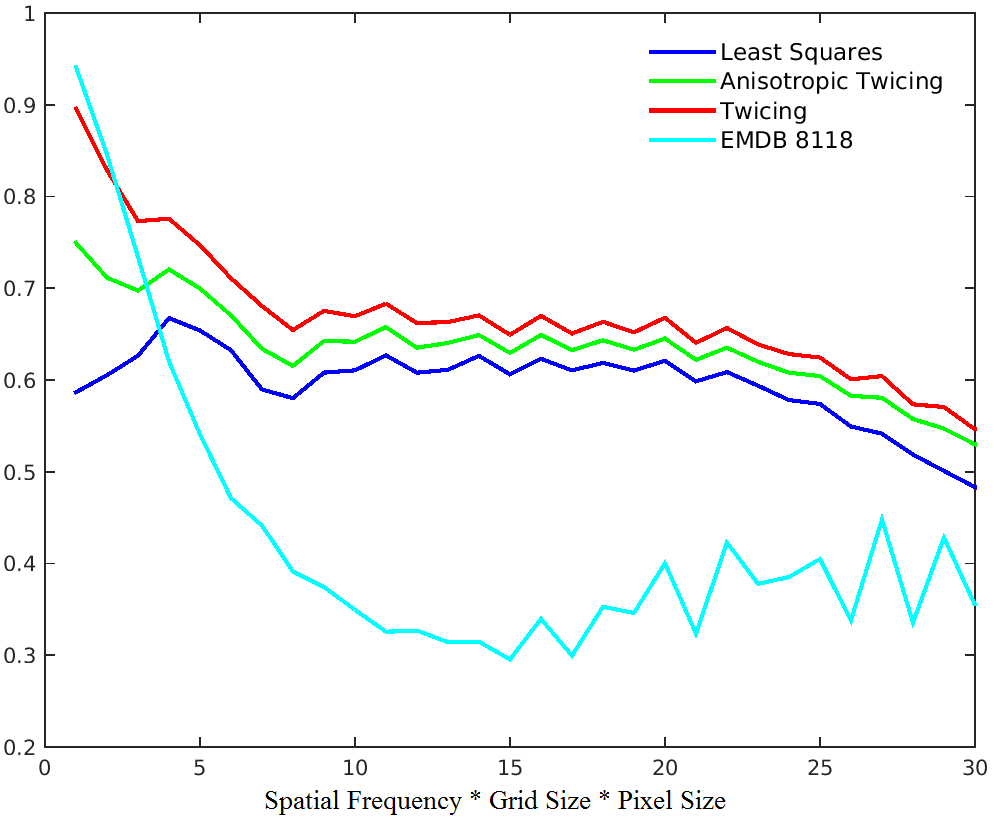}\\
(i) &(ii)\\
\end{tabular}
\label{Fig:fsc}
\end{figure}

\section{Conclusion}

The orthogonal retrieval problem in SPR  is akin to the phase retrieval problem \cite{Burvall2011, Liu2008, Pfeiffer2006}
in X-ray crystallography. In crystallography,  
the measured diffraction patterns contain information about the modulus of the
3D
Fourier transform of the structure but the phase information is missing and
needs to be obtained by other means. In crystallography,
the particle's orientations are known but
the phase of the Fourier coefficients is missing, while in cryo-EM, the projection images contain phase information but the orientations
of the
particles are missing. Kam's autocorrelation analysis for SPR leads to an orthogonal retrieval problem which is analogous to
the phase retrieval problem in crystallography. 
The phase retrieval problem is perhaps more challenging than
the orthogonal matrix retrieval problem in cryo-EM. In crystallography
each Fourier coefficient is missing its phase, while in cryo-EM only a single
orthogonal matrix is missing per several radial components. For each $l$, the unknown coefficient matrix $\bA_{\bl}$ is of size $S_l \times (2l+1)$, corresponding to $(2l+1)$ radial functions. Each $\bA_{\bl}$ is to be obtained from $\bC_{\bl}$, which is a positive semidefinite matrix of size $S_l \times S_l$ and rank at most $2l+1$. For $S_l>2l+1$, instead of estimating $S_l(2l+1)$ coefficients, we only need to estimate an orthogonal matrix in $\text{O}(2l+1)$ which allows $l(2l+1)$ degrees of freedom. Therefore there are $(S_l-l)(2l+1)$ fewer parameters to be estimated.

It is important to note that the main requirement 
for OE to succeed is that there are sufficiently many images to estimate the 
covariance matrix to the desired level of accuracy, so it has a much greater 
chance of success for homology modeling from very noisy images than other 
ab-initio methods such as those based on common lines, which fail at very high noise 
levels.

In this paper, we find a general magnitude correction scheme for the class of
`phase-retrieval' problems, in particular, for Orthogonal Extension in cryo-EM.
The magnitude correction scheme is a generalization of `twicing' that is
commonly used in molecular replacement. We derive an asymptotically unbiased estimator and demonstrate 3D homology modeling using OE with synthetic and experimental datasets. We foresee this method as a good way to provide models to initialize refinement, directly from experimental images without performing class averaging and orientation estimation in cryo-EM and XFEL.

While Anisotropic Twicing outperforms least squares and twicing for synthetic data, the three estimation methods have similar performance for experimental data. One possible explanation is that the underlying assumption made by all estimation methods that $\bC_{\boldsymbol{l}}$ are noiseless as implied by imposing the constraint $\bC_{\boldsymbol{l}}=\bA_{\boldsymbol{l}} \bA_{\boldsymbol{l}}^*$, is violated more severely for experimental data. Specifically, the $\bC_{\boldsymbol{l}}$ matrices are derived from the 2D covariance matrix of the images, and estimation errors are the result of noise in the images, finite number of images available, non-uniformity of viewing directions, and imperfect estimation of individual image noise power spectrum, contrast transfer function, and centering. These effects are likely to be more pronounced in experimental data compared to synthetic data. As a result, the error in estimating the $\bC_{\boldsymbol{l}}$ matrices from experimental data is larger. The error in the estimated $\bC_{\boldsymbol{l}}$ can be taken into consideration by replacing the constrained least squares problem (1) with the regularized least squares problem
 
\begin{equation}\label{eq:future_cl}
\min_{\bA} \lVert\bA-\bB\rVert_F^2+\lambda\lVert\bC-\bA\bA^*\rVert_F^2
\end{equation}
where $\lambda>0$ is a regularization parameter that would depend on the spherical harmonic order $l$. A comprehensive analysis of~\eqref{eq:future_cl} and its application to experimental datasets will be the subject of future work.     
 
\section{Acknowledgment}
We are indebted to Garib Murshudov for motivating this work with his suggestion to explore unbiased estimators. We would like to thank Xiuyuan Cheng and Zhizhen Zhao for many helpful 
discussions about this work. We are grateful for  discussions about 
experimental datasets with Daniel Asarnow, Xiaochen Bai, Adam Frost, Yuan Gao, David Julius, Eugene Palovcak, Yoel Shkolnisky, and Fred Sigworth. The authors were partially supported by Award Number R01GM090200 from the NIGMS, FA9550-12-1-0317 from AFOSR, the Simons Investigator Award and the Simons Collaboration on Algorithms and Geometry, and the Moore Foundation Data-Driven Discovery Investigator Award.

\bibliographystyle{iucr}
\bibliography{iucr1}

\newpage

\section{Appendix: Proof of Theorem~\ref{thm:main}}
\label{sec:proof}
\subsection{Explicit expression of $\hat{\bA}_{\text{LS}}$}


Since $\hat{\bA}_{\text{LS}}$ is independent of the choice of $\bF$ in the algorithm, we may assume that $\bF=\bA$ without loss of generality. Let $\bE=\bA-\bB$, then by assumption, $\bE$ is fixed, and 
\begin{align*}
\bV_0\bU_0^*=&(\bV_0\bSigma_0\bU_0^*)(\bU_0\bSigma_0^{-1}\bU_0^*)=(\bV_0\bSigma_0\bU_0^*)[\bU_0\bSigma_0^{2}\bU_0^*]^{-0.5}\\=&\bA^*(\bA-\bE)[(\bA-\bE)^*\bA\bA^*(\bA-\bE)]^{-0.5}.
\end{align*}
Therefore,
\begin{equation}\label{eq:hatX}
\hat{\bA}_{\text{LS}}=\bA\bA^*(\bA-\bE)[(\bA-\bE)^*\bA\bA^*(\bA-\bE)]^{-0.5}.
\end{equation}


Applying \eqref{eq:hatX}, we may simplify $\hat{\bA}_{\text{LS}}$ further as follows:
\begin{align}
\hat{\bA}_{\text{LS}}=&(\bA-\bE)^{*\,-1}(\bA-\bE)^*\bA\bA^*(\bA-\bE)[(\bA-\bE)^*\bA\bA^*(\bA-\bE)]^{-0.5}
\nonumber\\=&(\bA-\bE)^{*\,-1}[(\bA-\bE)^*\bA\bA^*(\bA-\bE)]^{0.5}.\label{eq:expression1}
\end{align}
Since $\bA\bA^*=\bC$, we may assume that the SVD decomposition of $\bA$ be given by $\bA=\bU\bSigma\bV^*$, where $\bSigma=\diag(\sigma_1, \sigma_2, \cdots, \sigma_D)$. Let $\bE_0=\bU^*\bE\bV$, then applying the derivative of matrix inversion we have
\begin{align}\nonumber
(\bA-\bE)^{*\,-1}= &\bA^{*\,-1} + \bA^{*\,-1}\bE^* \bA^{*\,-1} + O(\|\bE\|_F^2)\\\nonumber=& \bU\bSigma^{-1}\bV^* + (\bU\bSigma^{-1}\bV^*)\bE^*(\bU\bSigma^{-1}\bV^*)  + O(\|\bE\|_F^2)\\=&  \bU\bSigma^{-1}\bV^* +\bU\bSigma^{-1}\bE_0^*\bSigma^{-1}\bV^* + O(\|\bE\|_F^2). \label{eq:expression2}
\end{align}
We also have
\begin{align*}
(\bA-\bE)^*\bA=\bA^*\bA-\bE^*\bA=\bV\bSigma^2\bV^*-\bE^*\bU\bSigma\bV^*=\bV[\bSigma^2-\bE_0^*\bSigma]\bV^*
\end{align*}
and similarly, $\bA^*(\bA-\bE)=\{\bV[\bSigma^2-\bE_0^*\bSigma]\bV^*\}^*=\bV[\bSigma^2-\bSigma\bE_0]\bV^*$.
Then
\begin{align}
\nonumber\{(\bA-\bE)^*\bA\bA^*(\bA-\bE)\}^{0.5}
=&\{\bV[\bSigma^2-\bE_0^*\bSigma][\bSigma^2-\bSigma\bE_0]\bV^*\}^{0.5}
\\\nonumber
=&\bV\{[\bSigma^2-\bE_0^*\bSigma][\bSigma^2-\bSigma\bE_0]\}^{0.5}\bV^*
\\=&\bV[\bSigma^4-\bE_0^*\bSigma^3-\bSigma^3\bE_0+O(\|\bE\|_F^2)]^{0.5}\bV^*\label{eq:expression3}
\end{align}
Applying Lemma~\ref{lemma:deri}, we have that 
\begin{equation}\label{eq:expression4}
[\bSigma^4-\bE_0^*\bSigma^3-\bSigma^3\bE_0+o(\|\bE\|_F)]^{0.5}
=\bSigma^2 + \bZ + o(\|\bE\|_F),
\end{equation}
where the $ij$-th entry of $\bZ$ is given by \[
\bZ_{ij}=-\frac{\bE_{0,ji}^*\sigma_j^3+\bE_{0,ij}\sigma_i^3}{\sigma_i^2+\sigma_j^2}.
\]
Combining \eqref{eq:expression1}-\eqref{eq:expression4}, we have
\begin{align}
\nonumber\hat{\bA}_{\text{LS}}=&[\bU\bSigma^{-1}\bV^* +\bU\bSigma^{-1}\bE_0^*\bSigma^{-1}\bV^*] \bV[\bSigma^2 + \bZ]\bV^* +o(\|\bE\|_F)
\\\nonumber=&\bA + [\bU\bSigma^{-1}\bV^*]\bV\bZ\bV^*+\bU\bSigma^{-1}\bE_0^*\bSigma^{-1}\bV^*\bV \bSigma^2 \bV^*+o(\|\bE\|_F)
\\
=&\bA+\bU[\bSigma^{-1}\bZ+\bSigma^{-1}\bE_0^*\bSigma]\bV^*+o(\|\bE\|_F),
\label{eq:expression5}\end{align}
and the $ij$-th entry of $[\bSigma^{-1}\bZ+\bSigma^{-1}\bE_0^*\bSigma]$ can be explicitly written down by
\begin{align*}
\sigma_i^{-1}\bZ_{ij}+\sigma_i^{-1}\bE_{0,ji}^*\sigma_j
=&-\frac{\bE_{0,ji}^*\sigma_j^3+\bE_{0,ij}\sigma_i^3}{\sigma_i(\sigma_i^2+\sigma_j^2)}+\bE_{0,ji}^*\frac{\sigma_j}{\sigma_i}
\\=&\bE_{0,ji}^*\frac{\sigma_i\sigma_j}{\sigma_i^2+\sigma_j^2}-\bE_{0,ij}\frac{\sigma_i^2}{\sigma_i^2+\sigma_j^2}.
\end{align*}

\subsection{Expectation when $\bV$ is uniformly distributed}
From the analysis in the previous section, we have
\[
\Expect(\hat{\bA}_{\text{LS}}-\bA)=\bU\,\Expect_{\bV}[\bSigma^{-1}\bZ+\bSigma^{-1}\bE_0^*\bSigma]\bV^*+o(\|\bE\|_F),
\]
when $\bV$ is uniformly distributed on the set of all orthogonal matrices or all unitary matrices.

Now let us check the $ik$-th entry of $\Expect_{\bV}[\bSigma^{-1}\bZ+\bSigma^{-1}\bE_0^*\bSigma]\bV^*$, which is
\begin{align*}
&\sum_{j}\Big[\bE_{0,ji}^*\frac{\sigma_i\sigma_j}{\sigma_i^2+\sigma_j^2}-\bE_{0,ij}\frac{\sigma_i^2}{\sigma_i^2+\sigma_j^2}\Big]\bV_{kj}^*\\
=&\sum_{j}\Big[(\sum_{m,n}\bU^*_{mj}\bE_{mn}\bV_{ni})^*\frac{\sigma_i\sigma_j}{\sigma_i^2+\sigma_j^2}-\sum_{m,n}\bU^*_{mi}\bE_{mn}\bV_{nj}\frac{\sigma_i^2}{\sigma_i^2+\sigma_j^2}\Big]\bV_{kj}^*.\end{align*}
Applying the facts that for real-values matrices $\bV$ we have
\[
\Expect_{\bV}\bV_{ij}\bV_{mn}=\begin{cases} \frac{1}{D},\,\,\text{if $i=m$ and $j=n$}\\
0,\,\,\text{otherwise},
\end{cases}
\]
and for complex-valued matrices $\bV$, $\Expect_{\bV}\bV_{ij}\bV_{mn}=0$ for all $(i,j,m,n)$, and
\[
\Expect_{\bV}\bV_{ij}\bV_{mn}^*=\begin{cases} \frac{1}{D},\,\,\text{if $i=m$ and $j=n$}\\
0,\,\,\text{otherwise},
\end{cases}
\]
its expectation is given by
\begin{align*}&\frac{1}{D}\Big\{\sum_{m}(\bU_{mi}^*\bE_{mk})^*\frac{1}{2} - \sum_{m,j}\bU_{mi}^*\bE_{mk}\frac{\sigma_i^2}{\sigma_i^2+\sigma_j^2}\Big\}\\
=&\begin{cases}\frac{1}{D}\Big\{\frac{1}{2}[\bU^*\bE]_{ik} - [\bU^*\bE]_{ik}\sum_{j}\frac{\sigma_i^2}{\sigma_i^2+\sigma_j^2}\Big\},\,\,\,&\text{when $\bV$ is real-valued,}\\
-\frac{1}{D} [\bU^*\bE]_{ik}\sum_{j}\frac{\sigma_i^2}{\sigma_i^2+\sigma_j^2},\,\,\,&\text{when $\bV$ is complex-valued.}
\end{cases}
\end{align*}
Combining these elementwise expectations into a matrix, $\Expect_{\bV}[\bSigma^{-1}\bZ-\bSigma^{-1}\bE_0^*\bSigma]\bV^*=-\bT\bU^*\bE$. Therefore, we have
\begin{equation}
\Expect(\hat{\bA}_{\text{LS}}-\bA)=\bU\,\Expect_{\bV}[\bSigma^{-1}\bZ-\bSigma^{-1}\bE_0^*\bSigma]\bV^*
=-\bU\bT\bU^*\bE+o(\|\bE\|_F).
\end{equation}

\subsection{Lemmas}
\begin{lemma}\label{lemma:deri}
For a diagonal matrix $\bX=\diag(x_1,x_2,\cdots, x_D)$,  the $ij$-th entry of
$(\bX+\bE)^{0.5}-\bX^{0.5}$ is given by
\[
[(\bX+\bE)^{0.5}-\bX^{0.5}]_{ij}= [\bE_{ij}\cdot \frac{1}{x_i^{0.5}+x_j^{0.5}}]+
o(\|\bE\|_F).
\]
\end{lemma}
\begin{proof}
The proof is based on the following observation: if $\bY$ is diagonal and $\bC$
is small, 
\[
(\bY+\bC)^2-\bY^2=  \bY\bC+\bC\bY +\bC^2 = [\bC_{ij}\cdot (y_i+y_j)]
+o(\|\bC\|_F),
\]
where $[\bC_{ij}\cdot (y_i+y_j)]$ denotes a matrix of $D\times D$, with $ij$-th
entry given by $\bC_{ij}\cdot (y_i+y_j)$.

Then the lemma is proved by applying this observation to $\bY=\bX^{0.5}$ and
$(\bY+\bC)^2=\bX+\bE$:
\begin{align*}
\bE=&\big[[(\bX+\bE)^{0.5}-\bX^{0.5}]_{ij}\cdot
(x_i^{0.5}+x_j^{0.5})\big]+o(\|(\bX+\bE)^{0.5}-\bX^{0.5}\|_F).
\end{align*}
\end{proof}
\end{document}